\documentclass[runningheads]{llncs}
\usepackage[T1]{fontenc}
\usepackage{graphicx}

\usepackage{latexsym}
\usepackage{amssymb}
\usepackage{amsmath}
\usepackage{booktabs}
\usepackage{enumitem}
\usepackage{color}

\usepackage{amssymb}    
\usepackage{mathtools}  
\usepackage{tikz}          
\usepackage{diagbox}       
\usepackage{upgreek}    
\usepackage{stmaryrd}   
\usepackage{capt-of}    
\usepackage{hhline}

\newcommand{\BibTeX}{B\kern-.05em{\sc i\kern-.025em b}\kern-.08em\TeX}

\newcommand{\ff}{\mbox{\textit{ff}}}    

\newcommand{\mUnacc}{\Delta}
\newcommand{\mAcc}{I}

\DeclareRobustCommand{\DirectNESS}{(\tikz[baseline=-\the\dimexpr\fontdimen22\textfont2\relax,inner sep=0pt] \draw[dash pattern={on 4.5pt off 4.5pt}](0,0) -- (5mm,0);)}
\DeclareRobustCommand{\NESS}{(\tikz[baseline=-\the\dimexpr\fontdimen22\textfont2\relax,inner sep=0pt] \draw[dash pattern={on 0.84pt off 2.51pt}](0,0) -- (5mm,0);)}
\DeclareRobustCommand{\actual}{(\tikz[baseline=-\the\dimexpr\fontdimen22\textfont2\relax,inner sep=0pt] \draw[line width=0.75](0,0) -- (5mm,0);)}

\begin{document}
\title{An action language-based formalisation of an abstract argumentation framework}
\titlerunning{Action language for argumentation}
%
\author{Yann Munro\inst{1} \and
Camilo Sarmiento \inst{1} \and
Isabelle Bloch \inst{1} \and
Gauvain Bourgne\inst{1} \and
Catherine Pelachaud\inst{2} \and
Marie-Jeanne Lesot\inst{1}}
\authorrunning{Y. Munro et al.}
%
\institute{Sorbonne Université, CNRS, LIP6, F-75005 Paris, France \\
\email{firstname.surname@lip6.fr} \and
CNRS, Sorbonne Université, ISIR, F-75005 Paris, France \\
\email{firstname.surname@sorbonne-universite.fr}}
\maketitle              
\begin{abstract}
An abstract argumentation framework is a commonly used formalism to provide a static representation of a dialogue. However, the order of enunciation of the arguments in an argumentative dialogue is very important and can affect the outcome of this dialogue. In this paper, we propose a new framework for modelling abstract argumentation graphs, a model that incorporates the order of enunciation of arguments. By taking this order into account, we have the means to deduce a unique outcome for each dialogue, called an extension. We also establish several properties, such as termination and correctness, and discuss two notions of completeness. In particular, we propose a modification of the previous transformation based on a "last enunciated last updated" strategy, which verifies the second form of completeness.

\keywords{Abstract Argumentation Framework  \and Action Description Language \and Temporality.}
\end{abstract}
\section{Introduction}

The notion of dialogue, defined as an exchange between at least two agents, either real or virtual, is a central element during an interaction. No matter which type of dialogue game is played~\cite{walton1995commitment}, elements called \emph{arguments} are traded one after another between the participants following a set of rules. Formally, these dialogues can be modelled using the abstract argumentation framework~(AAF) introduced by Dung~\cite{dung_acceptability_1995}. Indeed, AAF provides tools for both creating a dialogue system~\cite{black2021argumentation} and reasoning on it. Using these arguments and a binary relation called the attack relationship, it is possible to represent and reason about contradicting information. For that, sets of arguments that can be accepted together, called extensions, are identified~\cite{baroni2011introduction}. In addition, explanations on the reasons why such sets are accepted or not can be returned~\cite{cyras_survey,vassiliades2021argumentation}. However, in the classical version of formal argumentation, there is no focus on the order in which arguments are stated. Since this is very important to establish the causal relations in the dialogue, a notion of temporality has to be included in order to better model a dialogue. Note that we observe the dialogue and do not take part in it. Thus, in contrast to argumentation based dialogue systems~\cite{black2021argumentation}, we adopt a descriptive approach to model and reason on an AAF. Moreover, we are using an abstract argumentation framework as an input. It means that arguments are abstract objects and the attack relation is only a way to model conflicts between them. Doing so, in contrast to the works done in argumentation theory~\cite{walton1996argument,walton2008argumentation}, we are focusing neither on the argument structure nor on the nature of the attack.

In this paper we propose to investigate the use of a Labelled Transition System (LTS) to include the enunciation order of arguments in the dialogue formalisation. We do not consider a LTS as a language for reasoning about action and change, but rather as a general way of adapting to a dynamic environment. 
Thus, in practice we consider an action description language (ADL) to formalise our work. In particular, we use the ADL developed in~\cite{sarmiento_action_2022-1} due to its well adapted tools for reasoning about causality, which is the first step towards explanations.
Our contributions are as follows: first a new framework is proposed to model and reason on dialogues based on an ADL and following abstract argumentation principles. This new method allows the order in which arguments are presented to be included in the model:
we establish a set of rules, inspired from argumentation labellings, to model the argumentative process, leading to a unique outcome, i.e. the acceptability status of each argument. In contrast to discussion games for a given semantic~\cite{caminada2017argumentation} where a dialogue is built a posteriori to find and justify the membership of an argument to the studied semantics, we model the actual dialogue in real time and update the acceptability thanks to these defined rules. Secondly, we study the formal properties of this method. In particular, we prove its termination and its correctness with respect to AAF 
as well as a discussion about two notions of completeness.
Finally, we propose a modification of the previous transformation based on a ``last enunciated last updated" strategy. Thanks to that, our second version of the notion of completeness is satisfied.


\section{Related work}

The main problem addressed in this work is the modelling of the order of enunciation of arguments in a dialogue in a formalism that is suitable for reasoning about the causality of the acceptability status of an argument. Both questions of temporality and causality have been studied independently in the context of abstract argumentation, while we aim to do both together.

\subsubsection{Temporality in AAF.} Barringer et al.~\cite{barringer2005temporal,barringer2012temporal}
propose different scales to include temporality in abstract argumentation: the ``object level" is focusing on the temporal evolution of each component of the graph while at the ``meta level" a snapshot of the whole system is taken at each time step. One of the solutions they suggested, without formalising it, is to include the action of adding arguments. This idea has been taken up in work aimed at modelling persuasion or negotiation between different agents. In particular, Arisaka and Satoh~\cite{arisaka2018abstract} propose an extension of AAFs called abstract persuasion argumentation (APA). These authors, in addition to adding a ternary relation modelling the persuasion process, define a transition operator that allows moving from one state, i.e. an APA, to the next one by adding or removing arguments and relations. In our paper, a possible formalisation of this idea is also proposed, based on an ADL. It gives both an ``object level" (with the traces) and ``meta level" (with the argumentative states) description of the temporality in the argumentation framework. However, in contrast to their approach where cycles and temporality are two separated concerns, we propose to make use of the temporality to deal with cycles which, to the best of our knowledge, has not been done before.


\subsubsection{Rewriting of the AAF.} Another approach to tackle the temporality issue is to extend the argumentation graph using logical languages capable of modelling this notion of temporality. The YALLA logical formalism~\cite{de2016argumentation} proposes to use revision or belief change operators to update an argumentation system. It offers a highly expressive language that allows finding which attack relations or arguments should be removed or added in order to achieve a given acceptability status at the subsequent time step. This reasoning can be used to build an explanation. However the number of terms manipulated by YALLA increases exponentially with the number of arguments, which is not suited to model a whole dialogue. The approach proposed in~\cite{doutre2017dynamic} proposes to use a dynamical propositional language to model the argumentation graph and its evolution across the dialogue. Nevertheless, it does not include a model of the notion of causality.

\subsubsection{Sequence of AAF.} Recently, Kampik et al.~\cite{kampik2024change} proposed to create one argumentation graph at each time step in the context of a quantitative bipolar framework (QBAF) and then derive different types of explanations to explain the changes between two graphs. In contrast to AAF, in a QBAF the acceptability of an argument is a numerical value computed using an acceptability function. Thus, a first difference with respect to our proposal lies in the nature of what is explained. Secondly, the focus of Kampik et al.'s work is to explain the changes between two graphs and not necessarily to model the temporality of a dialogue. Finally, unlike this work, our work is based on an ADL that provides the causal foundations necessary for the explanation process.

\subsubsection{Causality in argumentation.} 
Bengel et al.~\cite{bengel2022argumentation} use a knowledge-based formalism to define counterfactual reasoning for structured argumentation. Their definition is the direct application of Pearl's definition for causal model~\cite{pearl2000models}, similar to what is called a \textit{but-for} test. However, in some scenarios such as over-determination, this test does not reflect the human intuition of what is really the cause of what happened~\cite{menzies_counterfactual_2020}. Bochman in~\cite{bochman2005propositional}, extended in~\cite{bochman2021logical}, establishes an equivalence between abstract argumentation and a causal reasoning system.
While these works take inspiration from the social science literature, they focus on classical AAFs without temporality and so do not take the order of enunciation into account.


While there is work on incorporating a notion of temporality or studying causality, to the best of our knowledge none has provided a framework to do both simultaneously, which is the objective of this work. 

\section{Preliminaries}
\label{sec:prelim}

This section recalls the basics of Dung's AAF~\cite{dung_acceptability_1995} and illustrates them with an example. It then recalls the formal aspects of the ADL proposed in~\cite{sarmiento_action_2022-1}.

\subsection{Abstract Argumentation Framework (AAF)}
\label{sec:AAF}

An~$AAF$ is a couple~${AF = (A,\mathcal{R})}$, where~$A$ is a finite set of arguments and~$\mathcal{R}$ is a set of \emph{attacks} defined as a binary relation on~$A \times A$. Argument~$x\in A$ attacks~$y\in A$ if~$(x,y)\in \mathcal{R}$. As $\mathcal{R}$ is a binary relation with a finite support, $AF$ can be represented as a graph, called the argumentation graph.

For a given argumentation graph, an \emph{extension} is a subset of arguments that can be accepted together, where acceptability relies on the following notions:  The \emph{set of direct attackers} of~$x \in A$ is denoted by~$Att_{x}=\{ y\in A \mid (y,x)\in \mathcal{R}\}$. A set~$S \subseteq A$ is \emph{conflict-free} if~${\forall (x,y) \in S^2}$, $(x,y)\notin \mathcal{R}$. An argument~$x \in A$ is \emph{acceptable} by~$S$ if~$\forall y\in 
Att_x, S \cap Att_y \neq \emptyset$. $S$ is an \emph{admissible set} if it is conflict-free and all its elements are acceptable by~$S$ itself.


In the case of cyclic argumentation graphs, different semantics for the extensions have been proposed.
Here, we focus on Dung's admissible based ones:
the complete semantics, denoted~$\Sigma_c$, defines a complete extension as an admissible set~$S$ where all the arguments acceptable by~$S$ are included in~$S$. The grounded semantics~$\Sigma_g$ defines the grounded extension as the unique minimal complete extension. The preferred semantics~$\Sigma_p$ defines a preferred extension as a maximal complete extension. Finally, the stable semantics~$\Sigma_s$ defines a stable extension as a conflict-free set~$S$ where all the arguments not included in~$S$ are attacked by~$S$. Note that such an extension does not always exist.

Still in the case of cyclic graphs, another equivalent approach is to assign a value to each argument: IN, OUT or UNDEC.
These values can be computed using a labelling process~\cite{baroni2011introduction} that outputs a total function~$Lab$. An extension is then derived by taking all the arguments labelled IN. Depending on the properties of the considered $Lab$ function, the obtained semantics differ.

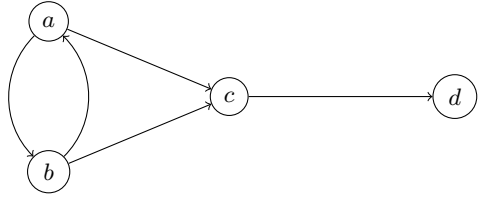
\begin{figure}[t]
    \centering
    \begin{tikzpicture}[scale=1,transform shape, node distance={30mm}, main/.style = {draw, circle}]

                \node[main] (1) {$a$};

                \node[main] (2) at (2.4,-1) {$c$};
                \node[main] (3) at (0,-2) {$b$};
                \node[main] (4) [right of=2] {$d$};

                \draw[->] (1) -- (2);
                \draw[->] (3) -- (2);
                \draw[->] (1) to [in = 135, out = -135] (3);
                \draw[->] (3) to [in = -45, out = 45] (1);
                \draw[->] (2) -- (4);

            \end{tikzpicture}
    \caption{
    Argumentation graph associated with Example~\ref{ex:IRM_ou_radio}.}
    \label{fig:ex_radio}
\end{figure}

\begin{example}
\label{ex:IRM_ou_radio}

Let us consider the following AAF: $AF = (\mathcal{A}, \mathcal{R})$ where $\mathcal{A} = \{a,b,c,d\}$ and $\mathcal{R} = \{(a,b),(b,a),(a,c),(b,c),(c,d)\}$. Its associated graph is presented in Figure~\ref{fig:ex_radio}.
Since $a$ attacks $b$ and $b$ attacks $a$, we can either accept one and reject the other, or choose none. This leads to either rejecting $c$ or not choosing, and then either $d$ is acceptable or undecided. Formally, there are three complete labellings for $AF$ corresponding to each scenario: $\mathcal{L}_1 = \{\{a,d\},\{b,c\},\{\emptyset\}\}$, ${\mathcal{L}_2 = \{\{b,d\},\{a,c\},\{\emptyset\}\}}$ and $\mathcal{L}_3 = \{\{\emptyset\},\{\emptyset\},\{a,b,c,d\}\}$ where the sets $IN$, $OUT$ and $UNDEC$ are $\mathcal{L} = \{IN, OUT, UNDEC\}$.

\end{example}

\subsection{Action Description Language (ADL)}
\label{sec:Action_language}

    
    Labelled Transition Systems generally define dynamical systems, describing the state of the world at any time step as a collection of variables that describe the properties of the world. Furthermore, they assert that the evolution of the world can be described as a series of states transitioning from one to another one as a result of a set of events. 
    In the sequel, $\mathbb{F}$ is the set of variables describing the state of the world, i.e.
    \emph{ground fluents} representing time-varying properties, and $\mathbb{E}$ is the set of variables describing transitions, 
    i.e. \emph{ground events} that modify fluents.
    
    \begin{definition}[Labelled Transition System (LTS)]\label{def:LTS}
        A \emph{Labelled Transition System} is a triplet~$\left<\mathbb{S},V,\tau\right>$ where~$\mathbb{S}$ is a set of states, $V$ is a function formally defined as~$V:\mathbb{S}\rightarrow2^\mathbb{F}$, and~$\tau$ is a set of \emph{labelled transition relations} $\tau\subseteq\mathbb{S}\times2^\mathbb{E}\times\mathbb{S}$.
    \end{definition}


    An ADL is a formalism describing a LTS~\cite{gelfond_action_1998}. Differences between ADLs are primarily determined by the conditions that the labelled transition relations~$\tau$ satisfy. One of our contributions is to adapt the one proposed in~\cite{sarmiento_action_2022-1} to AAF. We only describe below the elements of~\cite{sarmiento_action_2022-1} that are necessary to understand the adaptation we propose.
    
    A \emph{fluent literal} is either a fluent~$f\in \mathbb{F}$, or its negation~$\neg f$. The set of fluent literals in~$\mathbb{F}$ is denoted by~$Lit_{\mathbb{F}} = \mathbb{F}\cup\left\{\neg f \mid f\in \mathbb{F} \right\}$. The complement of a fluent literal $l$ is~$\overline{l}=\neg f$ if~$l=f$ and~$\overline{l}=f$ if~$l=\neg f$. We define \emph{state formulas}~$\mathcal{F}$ as formulas of fluent literals using classical logical operators.
    
    A set $S\subseteq Lit_{\mathbb{F}}$ is a \emph{state} if it is coherent $\left(\forall l\in S, \overline{l}\not\in S\right)$ and complete $\left(\forall f\in \mathbb{F},\right.$ $f\in S$ or $\left.\neg f \in S\right)$. In this formalism, the term complete should not be confused with the argumentation semantics~$\Sigma_c$. A state thus gives the value of each of the fluents describing the world. For~$\psi\in\mathcal{F}$, $S\vDash\psi$ is classically defined.
        
    
    An event~$e\in\mathbb{E}$ is an atomic formula. When considering causality, it is important to understand why an event~$e$ can occur and what are its consequences. Thus, three elements are needed to characterise each event: as detailed below, preconditions ($pre$) and triggering conditions ($tri$) are conditions that must be satisfied by a state~$S$ for the event to be triggered; effects indicate the changes to the state that are expected if the event occurs. The functions which associate preconditions, triggering conditions and effects with each event are respectively denoted by ${pre: \mathbb{E} \rightarrow \mathcal{F}}$, $tri: \mathbb{E} \rightarrow \mathcal{F}$, and~$e\ff: \mathbb{E} \rightarrow 2^\mathbb{F}$.
    
    The set $\mathbb{E}$ is partitioned into two  disjoint sets: 
    $\mathbb{A}$ contains the actions carried out by an agent and thus subjected to their volition; the set $\mathbb{U}$ contains the exogenous events which are triggered as soon as all the~$pre$ conditions are fulfilled, therefore without the need for an agent to perform them. Thus, for exogenous events,~$pre$ and~$tri$ are the same. By contrast, for actions, in addition to $pre$ conditions, $tri$ conditions also include the volition of the agent or some kind of external manipulation. To solve potential conflicts or to prioritise between events, a strict partial order~$\succ$ is introduced on $\mathbb{E}$, which ensures the triggering primacy of one event over another. 

    Sarmiento et al.~\cite{sarmiento_action_2022-1} focus on a path in the LTS, rather than the entire LTS. Formally, a path corresponds to a sequence of events and states where~$S(t)$ is the state we reach at a certain time~$t\in\mathbb{T}$ and~$E(t)$ the transitions that take us there
    where~$\mathbb{T}=\left\{0,\dots,N\right\}$. The \emph{initial state} is denoted by $S_0$.
    
    \begin{definition}[context $\kappa$]\label{def:context}
        A \emph{context}, denoted as $\kappa$, is an octuple $(\mathbb{E},\mathbb{F},pre, tri,$ $e\ff$, $S(0),\succ ,\mathbb{T})$, where $\mathbb{E}$, $\mathbb{F}$, $pre$, $tri$, $e\ff, S(0)$, $\succ$, and $\mathbb{T}$ are as defined above.
    \end{definition}
    
    As there is no information about when actions are performed in~$\kappa$, it is possible  to have more than one valid path for a given context~$\kappa$. In deterministic cases like ours, adding a set of timed actions~$\sigma\subseteq\mathbb{A}\times\mathbb{T}$ which models the volition of agents as an input, called \emph{scenario}, leads to a unique valid path satisfying: ${(i)~S(0)=S_0}$, ${(ii)~\forall t\in\mathbb{T},~\left(S(t),E(t),S(t+1)\right)\in\tau}$, and ${(iii)~\forall t\in\mathbb{T}, \forall e\in E(t)},$ ${e \in \mathbb{A} \Leftrightarrow (e,t) \in \sigma}$. From this unique path the event trace and state trace, denoted by~$\tau_{\sigma,\kappa}^e$ and~$\tau_{\sigma,\kappa}^s$, respectively, can be extracted.

    Another reason for using the ADL by Sarmentio et al.~\cite{sarmiento_action_2022-1} is its formal definition of actual causality, an essential building block of explanations~\cite{miller_explanation_2018}. Since the purpose of this work is to introduce the formalism, the discussion about causal reasoning and the process leading to explanations is left for future work.

\section{From AAF to ADL}
\label{sec:translation}
   
    This section presents our first contribution: the use of an ADL to enrich the AAF by taking into account the order of enunciation of arguments. Instead of having only a couple~$(A,\mathcal{R})$, the input is a couple~$(\Delta,\mathcal{R})$, where $\Delta$ is a \emph{dialogue}:
    ${\Delta = \{(a,o) \mid (a,o) \in A \times \mathbb{N}\}}$, where each argument~$a \in A$ is associated to its order of enunciation,~$o \in \mathbb{N}$.
    

    \subsection{Instantiating the Context}
    \label{sec:contextTranslation}

        To formalise an AAF in the ADL described in Section~\ref{sec:Action_language}, there are three main points to study: the fluents, the events, and the priority rules.
        
        \subsubsection{Describing the dialogue: the fluents.} Let us first define the variables necessary to describe the world, i.e. the AAF. These variables correspond to the fluents $\mathbb{F}$. As introduced in Section~\ref{sec:AAF}, there are two elements to consider: the arguments and the attack relation. To describe an argument~$x$, we create three fluents. First, $p_x \in\mathbb{F}$ represents the fact that~$x$ is present in the graph. Then, as reminded in Section~\ref{sec:AAF}, for cyclic argumentation framework, the acceptability status of arguments can take three values. To obtain this three-valued logic, we propose to use two fluents: $i_x$ (resp. $o_x$) represents the fact that~$x$ is in (resp. out). 
        Combining them makes it possible to get the three different acceptability values: ${a_x = i_x \wedge \neg o_x}$ corresponding to be acceptable or IN for labellings, ${na_x = \neg i_x \wedge o_x}$ corresponding to be not acceptable or OUT, and ${undec_x = \neg i_x \wedge \neg o_x}$ corresponding to the undetermined state or UNDEC. The fourth combination $i_x \wedge o_x$ is considered impossible. Another option could have been to create these three fluents $a_x,na_x$ and $undec_x$ directly. Nevertheless, to deal with cycles properly, we want to separate the acceptability status updates as much as possible. For instance, when an acceptable argument is attacked by another acceptable or undecided argument then it becomes undecided. In our formalism, it simply means being attacked by an argument that is \emph{not out}. Moreover, because the fourth combination is impossible, $x$ is acceptable is the same as $i_x$ holds, and $x$ is unacceptable is the same as $o_x$ holds. Therefore, creating those three fluents is not useful with this decomposition.

        The last element that is needed to model the dialogue is the attack relation~$\mathcal{R}$ between the arguments: we use the fluent $cA_{y,x}\in\mathbb{F}$ to model that argument~$y$ can attack argument~$x$.

        \addtocounter{example}{-1}
        \begin{example}(continued)
            If we apply this transformation on the previous example, we get that $\mathbb{F} = \mathbb{F}_{\mathcal{A}}\cup \mathbb{F}_{\mathcal{R}}$ where: \begin{itemize}
                \item $\mathbb{F}_{\mathcal{A}} = \{p_a,i_a,o_a,p_b,i_b,o_b,p_c,i_c,o_c,p_d,i_d,o_d \}$
                \item $\mathbb{F}_{\mathcal{R}} = \{cA_{a,b},cA_{b,a},cA_{a,c},cA_{b,c},cA_{c,d}\}$
            \end{itemize}
            
        \end{example}
        
        
        \subsubsection{Updating the argumentation graph: the events.}
        During a dialogue, the only deliberate action we consider is to enunciate an argument, which leads to ${\mathbb{A} = \{enunciate_x \mid x \in A \}}$.
        When performed, $x$ becomes present and is acceptable by default. It is also worth noting that if an argument is already present, enunciating it again has the effect of resetting its acceptability value to being acceptable. This way it is possible to repeat an argument that was previously rejected or undecided to reconsider its status. 
        Formally: \begin{center}
            $pre(enunciate_x)\equiv \top; \hspace{0.3cm}e\ff(enunciate_x)\equiv p_x\wedge i_x \wedge \neg o_x$
        \end{center} 

        \addtocounter{example}{-1}
        \begin{example}(continued)
            If we continue the transformation on the example, we get $\mathbb{A} = \{enunciate_a,enunciate_b,enunciate_c,enunciate_d\}$.
            
        \end{example}
        
        After an enunciation, the acceptability statuses of all present arguments are recursively updated according to rules that depend on the considered semantics~$\Sigma$ (see Section~\ref{sec:AAF}).

        A crucial question is that of the termination of this recursive process, meaning that a state where no acceptability rules can be triggered anymore is reached. We call such a state a \emph{$\Sigma$-argumentative state}, its definition depends on the considered semantics. In this section, we focus on the complete semantics~$\Sigma_c$, discussing other ones along with the technical difficulties they raise in Section~\ref{sec:discu_cycles}. We exploit the characterisation of complete labelling established in~\cite{caminada2011judgment}: a labelling is complete iff for an argument~$x$, $x$ is labelled IN iff all its attackers are labelled OUT, and it is labelled OUT iff it has at least one attacker that is labelled IN. This principle leads to the following definition:
        
        
        \begin{definition}[$\Sigma_c$-argumentative state]\label{def:complete_state}
            A state~$S(t)$ is a \emph{$\Sigma_c$-argumentative state} if:\\ 
            i) $\forall x, [S(t)\models p_x \wedge i_x \wedge \neg o_x \Leftrightarrow S(t)\models p_x \wedge \left(\bigwedge_{y} (cA_{y,x} \wedge \neg i_y \wedge o_y) \vee \neg cA_{y,x} \right)]$;\\
            ii) $\forall x, [S(t)\models p_x \wedge \neg i_x \wedge o_x \Leftrightarrow \exists y, S(t)\models  p_x \wedge p_y \wedge cA_{y,x} \wedge i_y \wedge \neg o_y]$.
        \end{definition}

        To reach a $\Sigma_c$-argumentative state after an argument is enunciated, acceptability updates are triggered automatically. 
        In abstract argumentation, an argument is acceptable only if it is unattacked or attacked only by unacceptable arguments. It is also enough for one of the attackers to be acceptable to make the attacked argument unacceptable. Finally, an argument with no acceptable attackers and at least one attacker that is not unacceptable is undecided. Because there might be cycles in the argumentation graph, we have to divide this update mechanism as much as possible. To do so, we use four types of updates, represented with four exogenous events: respectively modelling the transitions from acceptable to undecided (${\mUnacc^{1}_{y,x}}$), from undecided to unacceptable (${\mUnacc^{2}_{y,x}}$), from unacceptable to undecided (${\mAcc^1_x}$) and from undecided to acceptable~(${\mAcc^2_x}$), i.e. $\mathbb{U} = \{{\mUnacc^{1}_{y,x}},{\mUnacc^{2}_{y,x}}, {\mAcc^1_x},{\mAcc^2_x}\}$. Separating the process of moving directly from acceptable to unacceptable and vice versa into two steps is a crucial element of the transformation we propose. Indeed, it avoids going round in circles due to cycles and is therefore a fundamental element in ensuring the termination property.
        
        \subsubsection{From acceptable to undecided.}We first study the rule that leads an acceptable argument to become undecided. Suppose that an argument $y$ can attack argument $x$, and that $y$ is not unacceptable i.e. undecided or acceptable while~$x$ is acceptable. Then, $x$ being attacked by an argument~$y$ which is not unacceptable, it cannot be acceptable anymore. Formally, the exogenous event ${\mUnacc^{1}_{y,x}}$ can be written as:
            \begin{center}
                $tri(\mUnacc^{1}_{y,x})\equiv p_x \wedge i_x\wedge p_y\wedge cA_{y,x} \wedge \neg o_y;$ 
                
                $e\ff(\mUnacc^{1}_{y,x})\equiv \neg i_x$\phantom{$\wedge i_x\wedge p_y\wedge cA_{y,x} \wedge \neg o_y;$}
            \end{center}
            
        
        \subsubsection{From undecided to unacceptable.} Now if an acceptable argument~$y$ attacks an undecided argument~$x$, then~$x$ becomes unacceptable. This leads to the second exogenous event ${\mUnacc^{2}_{y,x}}$:
            \begin{center}
                $tri(\mUnacc^{2}_{y,x})\equiv p_x \wedge \neg o_x\wedge p_y\wedge cA_{y,x} \wedge i_y;$ 
                
                $e\ff(\mUnacc^{2}_{y,x})\equiv o_x$\phantom{$\wedge i_x\wedge p_y\wedge cA_{y,x} \wedge \neg o_y;;$}
            \end{center}

        
        \subsubsection{From unacceptable to undecided.}Now, let us suppose that argument~$x$ is unacceptable and that because of some circumstances, all its attackers have become not acceptable i.e. \emph{not in}. Then, $x$ cannot be unacceptable anymore and becomes not unacceptable i.e. \emph{not out}. 
        It is represented by event~$\mAcc^1_x$: 
        \begin{center}
            $tri(\mAcc^1_x) \equiv p_x\wedge o_x\wedge \left(\bigwedge_{y} ( cA_{y,x} \wedge \neg i_y) \vee \neg cA_{y,x} \right);$ \vspace{-1.5mm}
            
            $e\ff(\mAcc^1_x)\equiv \neg o_x$\phantom{$\wedge o_x\wedge \left(\bigwedge_{y} cA_{y,x} \wedge \neg i_y \vee \neg cA_{y,x} ()\right);$}
        \end{center}

        \subsubsection{From undecided to acceptable.}Finally, if argument~$x$ is not acceptable and is attacked only by arguments that are unacceptable then~$x$ becomes acceptable. This is translated in the ADL by event~$\mAcc^2_x$: 
        \begin{center}
            $tri(\mAcc^2_x) \equiv p_x\wedge \neg i_x\wedge \left(\bigwedge_{y} ( cA_{y,x} \wedge o_y) \vee \neg cA_{y,x} \right);$ \vspace{-1.5mm}
            
            $e\ff(\mAcc^2_x)\equiv i_x$\phantom{$\wedge o_x\wedge \left(\bigwedge_{y} cA_{y,x} \wedge \neg i_y \vee \neg cA_{y,x} ()\right);;;$}
        \end{center}
        
        With these exogenous events, we have broken down the process of making arguments acceptable or unacceptable in two steps by using the undecided status. This is necessary for termination when dealing with cycles to avoid going round in circles. For example, in the case of a two-element cycle, both arguments are first undecided and then the update is stopped, rather than simultaneously becoming acceptable, then unacceptable, then acceptable, and so on and so forth. In the case of acyclic graphs, it would have been possible to merge the first and second rules of each type. 

        \addtocounter{example}{-1}
        \begin{example}(continued)
            Enumerating all the exogenous events for that example would be too long: there are two $\mAcc$ events per argument, and then for each argument there are six $\mUnacc$ events so a total of $32$ exogenous events. For example for argument $a$ we have $\mathbb{U} = \{\mAcc^1_a,\mAcc^2_a,\mUnacc^1_{b,a},\mUnacc^2_{b,a},\mUnacc^1_{c,a},\mUnacc^2_{c,a},\mUnacc^1_{d,a},\mUnacc^2_{d,a}\}$.
            
        \end{example}
        
        \subsubsection{Cycles and Priority rules.}
        The decomposition through the undecided  status requires defining priority conditions taking into account two elements.
        The first one is that the decomposed rules ($\mAcc^1$ and $\mAcc^2$, $\mUnacc^1$ and $\mUnacc^2$) have to be triggered in the right order. 
        This prevents the fact that an argument~$x$ is both \emph{in} and \emph{out} which is forbidden. Therefore, we prioritise the rules numbered $1$ that lead to the undecided status before possibly lifting the indecision. It can be seen as first focusing on acceptable and unacceptable arguments before handling undecided ones. 
        This is reflected by the two rules: $R_1 := \mAcc^1_x \succ  \mAcc^2_x$ and $R'_1 := \forall x,y,$ $\mUnacc^1_{y,x} \succ  \mUnacc^2_{y,x}$. 
        
        The second element is to handle cycles and in particular to guarantee the termination of our transformation. To do that, we use the following priority rules: $\forall x_1,y_1,x_2,y_2,$ ${R_2 := \mAcc^1_{x_1}  \succ  \mUnacc^2_{y_2,x_2}}$ and $R_3 := \mUnacc^1_{y_1,x_1} \succ  \mUnacc^2_{y_2,x_2}$. While these rules are stronger than the previous ones, the principle behind them is the same: focusing on acceptable and unacceptable arguments before handling undecided ones. For more details, please refer to the termination proof in the supplementary material.  
        Note that the latter rule $R_3$ includes the second one $R'_1$ presented. However, $R_3$ is a stronger version needed for dealing with cycles whereas $R'_1$ can be seen more as an argumentative one. For example, in the case of acyclic argumentation graphs, only $R'_1$ is needed.

        \addtocounter{example}{-1}
        \begin{example}(continued)
            Let us consider the transformation without the priority rules described above. Suppose we are in a state $S(t_0)$ in which only $a$ is present and acceptable. At this instant, argument $b$ is enunciated leading to the next state $S(t_0+1)$. Since $S(t_0+1) \models i_a \wedge \neg o_a \wedge i_b \wedge \neg o_b$, $\mUnacc^1_{a,b}, \mUnacc^2_{a,b}$ and $\mUnacc^1_{b,a}, \mUnacc^2_{b,a}$ are triggered. We get that $S(t_0+2) \models \neg i_a \wedge o_a \wedge \neg i_b \wedge o_b$. Therefore, $\mAcc^1_a,\mAcc^2_a,\mAcc^1_b,\mAcc^2_b$ are triggered leading to a state $S(t_0+3) \models i_a \wedge \neg o_a \wedge i_b \wedge \neg o_b$. In this scenario, an endless loop is triggered. This is prevented by the first set of two rules.

            In the case of odd-length cycles, this is not enough and the second set of rules is required.            
        \end{example}        
        
            
            

    \subsection{ADL Semantics Adapted to AAF}
    \label{sec:semanticsTranslation}
    
        Having an adapted~$\kappa$ for the argumentative framework, we now propose to modify the ADL semantics to produce traces that are representative of the dialogue. For this purpose, arguments are stated from $\Sigma_c$-argumentative states step by step in the order determined by the dialogue~$\Delta$.

        First, let us define what is considered a valid path in a LTS adapted to AAF. Defining this can be seen as setting the conditions that the labelled transition relation~$\tau$ must satisfy. The modifications to the conditions in the original ADL~\cite{sarmiento_action_2022-1} can be found in conditions 2.d and 2.e. These conditions specify that an action in the sequence can only be triggered if no exogenous event is triggered at the same time point, and that the event trace cannot be empty.

        \begin{definition}[valid path in an argumentative context]\label{def:semantics_arg}
            Given an argumentative context~$\kappa$, a sequence $E(-1),S(0),\dots,S(N+1)$ is a \emph{valid path} w.r.t. $\kappa$ if $S(0)=S_0$ and~$\forall t\in\mathbb{T},\left(S(t),E(t),S(t+1)\right)\in\tau$ satisfy:\\
                \phantom{abc}1. $S(t)\subseteq Lit_{\mathbb{F}}$ is a state.\\
                \phantom{abc}2. $E(t)\subseteq\mathbb{E}$ satisfies:\\
                    \phantom{abc2.}2.a $\forall e\in E(t)$, $S(t) \models pre(e)$;\\
                    \phantom{abc2.}2.b $\nexists (e,e')\in E(t)^2,~e\succ e'$;\\
                    \phantom{abc2.}2.c $\forall e\in \mathbb{U}$ s.t. $S(t) \models tri(e)$, then $e \in E(t)$ or ${\exists e'\in E(t), ~e'\succ e}$;\\
                    \phantom{abc2.}2.d If $\exists e\in E(t)\cap\mathbb{A}$, then $\forall e'\in\mathbb{U}$, $S(t)\not\models tri(e')$;\\ 
                    \phantom{abc2.}2.e $E(t)\not=\varnothing$;\\
                \phantom{abc}3. $S(t+1)=\{l\in S(t)|\forall e\in E(t),\overline{l}\not\in e\ff(e)\} \cup\{l\in Lit_{\mathbb{F}}|\exists e\in E(t),l\in e\ff(e)\}$. 
        \end{definition}

        As in the ADL proposed in~\cite{sarmiento_action_2022-1}, there is potentially more than one valid path for a given context~$\kappa$. The form of scenario~$\sigma$ used in this previous work to obtain a unique path is not convenient for the formalisation of a dialogue because it requires that the number of steps that each chain of admissibility update events will take is known in advance. Without this information, it is not possible to plan at which time the next argument should be stated. To solve this issue we introduce a set of ranked actions~$\varsigma\subseteq\mathbb{A}\times\mathbb{N}$,
        called \emph{sequence}. The input to obtain unique traces is no longer a scenario~$\sigma$ but a sequence~$\varsigma$.

        \begin{definition}[argumentative setting $\chi$]\label{def:arg_setting}
            The \emph{argumentative setting} of the ADL, denoted by~$\chi$, is the couple~$(\varsigma,\kappa)$ with~$\varsigma$ a sequence and~$\kappa$ a context.
        \end{definition}

        \begin{definition}[valid path given~$\chi$]\label{def:traces_arg}
            Given an argumentative setting $\chi=(\varsigma,\kappa)$, a valid path w.r.t.~$\kappa$ is \emph{valid w.r.t. to~$\chi$} if:\\
            \phantom{abc}1. ${\forall t\in\mathbb{T}}$, ${E(t)\subset\left(\left\{a|\exists o\in\mathbb{N}, (a,o)\in\varsigma\right\}\cup\mathbb{U}\right)}$;\\
            \phantom{abc}2. ${\forall \left(\left(e,o\right),\left(e',o'\right)\right)\in\varsigma^2}$ s.t. $o<o'$, \\ \phantom{abc.2.}$\exists t,t'$, $e\in E(t), e'\in E(t') \mbox{ and } t<t'$;\\
            \phantom{abc}3. ${\forall \left(\left(e,o\right),\left(e',o'\right)\right)\in\varsigma^2}$ s.t. $o=o'$, $\exists t$, $(e,e')\in E(t)^2$.
        \end{definition}
        
        Given a valid path given~$\chi$, its \emph{event trace}~$\tau_{\chi}^e$ is its sequence of events, its \emph{state trace}~$\tau_{\chi}^s$ is its sequence of states.

\section{Formal properties}
\label{sec:discu_cycles}

In this section, we first establish and discuss some formal properties of the transformation proposed in Section~\ref{sec:translation},
including the correctness and termination of this transformation. Detailed proofs are given as supplementary material. After that, we discuss two forms of completeness. Then we propose to explore a modification of the transformation according to a new strategy that may lead to one of these forms of completeness. We finally discuss other semantics.

\label{sec:modif_translation}

We first show that, although valid paths given~$\kappa$ are not unique, valid paths given~$\chi$ are, as well as the corresponding traces~$\tau_{\chi}^e$ and~$\tau_{\chi}^s$. 
    
\begin{proposition}\label{prop:unicity_valid_path}
    Given an argumentative setting $\chi=(\varsigma,\kappa)$, traces $\tau_{\chi}^e$, $\tau_{\chi}^s$ are unique.
\end{proposition}

This proposition is proved by contradiction (not detailed in this paper).

From now on, when reference is made to events and states, they will be those from the unique traces~$\tau_{\chi}^e$ and~$\tau_{\chi}^s$, respectively. Thus, the set of all events that actually occur at time point~$t$ is ${E^{\chi}(t)=\tau_{\chi}^e(t)}$. Following the same principle, the actual state at time point~$t$ is~$S^{\chi}(t)=\tau_{\chi}^s(t)$.

When defining the fluents, we wanted \emph{in} and \emph{out} to be impossible simultaneously. 
As considered there is no state in the trace state where an argument is both \emph{in} and \emph{out}, which shows the appropriateness of the proposed rules in this regard:

\begin{proposition}\label{lem:impossible_fluents}
    $\forall t, \forall x$, if $S^{\chi}(t)$ is a state then $S^{\chi}(t) \not\models i_x \wedge o_x$.
\end{proposition}

This is an important proposition as it establishes a characterisation of being acceptable and unacceptable. Indeed, if an argument is \emph{in}, it is \emph{not out}, so being \emph{in} means being acceptable, and the same goes for \emph{out} and being unacceptable.

The next result is a characterisation of a $\Sigma_c$-argumentative state as a state where no exogenous event is triggered: such a state can be seen as an equilibrium state until a deliberate action, i.e. a new argument enunciation, occurs. Another way to say it is that an argumentative state corresponds to a real step in the dialogue while non argumentative ones are intermediate steps used to model the argumentative updating process.

\begin{lemma}\label{lem:argu_state_tri}
    Let $S^{\chi}(t)$ be a state. The two following propositions are equivalent:

    (i) $S^{\chi}(t)$ is a $\Sigma_c$-argumentative state as in Definition~\ref{def:complete_state}.

    (ii) $\forall e \in \mathbb{U}$, $S^{\chi}(t) \not\models tri(e)$.
\end{lemma}

\label{sec:modif_translation_prop}

We now establish the {\bf termination and correctness} of the transformation we propose. First, as a $\Sigma_c$-argumentative state is a state where nothing happens unless an argument is enunciated and $S(0)$ is a $\Sigma_c$-argumentative state, to prove termination we prove that such a state is always reached in a finite number of steps. This is established in Theorem~\ref{prop:termination}.

\begin{theorem}[Termination]\label{prop:termination}
    Given a $\Sigma_c$-argumentative state $S^{\chi}(t)$ and $x \in A$, if $enunciate_x \in E^{\chi}(t)$, then $\exists t'>t \in \mathbb{T}$ such that $S^{\chi}(t')$ is a $\Sigma_c$-argumentative state.
\end{theorem}


Secondly, we establish the correctness of the formalisation we propose. For that, we define a notion of associated labelling $\mathcal{L}_t$ for any $t \in \mathbb{T}$:

\begin{definition}\label{def:associated_labelling}
    Let $S^{\chi}(t)$ be a state. The associated labelling on $Ar = \{x \in A \mid S^{\chi}(t) \models p_x \}$ is the function $\mathcal{L}_t$ s.t.: 
    
    (i) $in(\mathcal{L}_t) = \{x \in Ar \mid S^{\chi}(t) \models i_x \wedge \neg o_x \}$,
    
    (ii) $out(\mathcal{L}_t) = \{x \in Ar \mid S^{\chi}(t) \models \neg i_x \wedge o_x\}$,
    
    (iii) $undec(\mathcal{L}_t) = \{x \in Ar \mid S^{\chi}(t) \models \neg i_x \wedge \neg o_x \}$.
\end{definition}

As the definition of a $\Sigma_c$-argumentative state is built based on a characterisation of a complete labelling, we establish the correctness property in the following Theorem:

\begin{theorem}[Correctness]\label{prop:correctness}
    Let $S^{\chi}(t)$ be a $\Sigma_c$ argumentative state and $\mathcal{L}_t$ its associated labelling. 
    Then, $\mathcal{L}_t$ is a complete labelling of the argumentation framework $AF = (Ar, \mathcal{R})$ where $Ar = \{x \in A \mid S^{\chi}(t) \models p_x \}$ and $\mathcal{R} = \{ (y,x) \mid S^{\chi}(t) \models p_x \wedge p_y \wedge cA_{y,x} \}$. 
\end{theorem}

\noindent The argumentation framework $AF$ defined in Theorem~\ref{prop:correctness} is called the associated graph of $S^{\chi}(t)$.


Theorem~\ref{prop:termination}, in combination with Proposition~\ref{prop:unicity_valid_path}, states that given an order of enunciation, there is a unique final state and that this state is a $\Sigma_c$-argumentative state. Then, using Theorem~\ref{prop:correctness}, the labelling associated with this state is unique and so it gives a unique complete extension of the associated argumentation framework, i.e. the whole dialogue. Therefore, modelling a dialogue with the formalism we propose instead of an argumentation framework allows us, thanks to the order of enunciation, to output a unique complete extension i.e. a unique outcome for the dialogue.

\label{sec:modif_translation_discu}


\subsubsection{Two versions of completeness.} A first definition of completeness is, given a dialogue~$\Delta$ and an argumentative semantics~$\Sigma$, to be able to find all the extensions of this semantics. This means that there is no loss of information when using the formalism we propose to model the dialogue instead of an argumentative one. By contrast, we gain some insight into how such a state has been reached thanks to the traces. The problem with this completeness is that it erases most of the benefits of modelling the order of enunciation of arguments. Indeed, if no matter what the order is we require to find all the extensions, then the order only matters for the traces. We want to go a bit further and use this temporality to discriminate between extensions. 

Thus, we propose to consider another form of completeness: for each extension of a given semantics, there exists an order of enunciation such that the labelling associated with its final state is equal to this extension. Using this definition, we can find all the extensions and at least one order of enunciation that leads to such an extension. Then, if an agent wants to reach a certain decision i.e. a particular extension, finding an order that leads to it can be seen as a debating strategy. This is the type of completeness we are looking for.

\subsubsection{A new transformation.} In the current state of the transformation, this second form of completeness does not hold. Indeed, to handle graphs with odd-length cycles, the two priority rules $R_2$ and $R_3$ lead to set the acceptability status of the arguments to undecided before applying other rules. For that reason, in the case of an even-length cycle, the only final state that can be reached is the one where all arguments are undecided. To deal with that issue, 
we explore here the possibility of another  strategy.

For example, let us consider the case where an argument~$x$ that is attacked and also attacks other arguments is enunciated. Without any further hypotheses, event~$\mUnacc^1$ may be triggered and change the acceptability status of~$x$ to undecided. Here, we consider the ``last enunciated last updated" strategy (discussing the appropriateness of this strategy is beyond the scope of this paper) i.e. we first apply all of the enunciated argument effects and then update its own status, eventually triggering another chain of exogenous events. By doing so, all attackers of~$x$ may become unaccepted and so the status of~$x$ may not change. To do that, the context~$\kappa$ has to be modified. First, we need a new fluent~$l_x$ that is set to true when an argument has just been enunciated. This fluent means that this argument has to be updated last. Then, the action $enunciate_x$ becomes: 
\\ $pre(enunciate_x)\equiv \top; \hspace{0.1cm}e\ff(enunciate_x)\equiv p_x \wedge l_x \wedge i_x \wedge \neg o_x$

Then, in order the acceptability update of a just stated argument to occur last, we create two instances of each exogenous event, e.g. $\mUnacc^1_{y,x}$ is replaced by $\mUnacc^{1f}_{y,x}$ and $\mUnacc^{1l}_{y,x}$ such that: 
\begin{center}
    $tri(\mUnacc^{1f}_{y,x})\equiv p_x  \wedge \neg l_x \wedge i_x\wedge p_y\wedge cA_{y,x} \wedge \neg o_y;$ 
    $e\ff(\mUnacc^{1f}_{y,x})\equiv \neg i_x$\phantom{$ \wedge \neg l_x \wedge i_x\wedge p_y\wedge cA_{y,x} \wedge \neg o_y;$} 
    $tri(\mUnacc^{1l}_{y,x})\equiv p_x \wedge l_x \wedge i_x\wedge p_y\wedge cA_{y,x} \wedge \neg o_y;$ \phantom{;} 
    $e\ff(\mUnacc^{1l}_{y,x})\equiv \neg i_x \wedge \neg l_x $\phantom{$ \wedge l_x \wedge i_x\wedge p_y\wedge cA_{y,x};;;$} 
\end{center}

Along the same lines, from the three other exogenous events we create six exogenous events, three for the earliest arguments and three for the latest ones. We define $\mathbb{U} = \mathbb{U}^f \cup \mathbb{U}^l$. 

Finally, the last step is to update the priority rules with the new exogenous events. The four previous rules become eight. Moreover, to force the model to update the latest arguments when nothing else can happen, the following rules are needed: $\forall e,e' \in \mathbb{U}^f \times \mathbb{U}^l, e \succ  e'$. 

With this update of the transformation, it is possible to find an order of enunciation for each extension of the complete semantics in the case of  even-length and odd-length cycles, solving a problem for the completeness of the first transformation. 

\begin{theorem}[Completeness of the modified transformation]
    Let $AF$ be an AAF and $\mathcal{L}_{c}$ be a complete labelling on $AF$. Then $\exists \varsigma$ a sequence such that the associated labelling $\mathcal{L}_t$ of the final argumentative state $S^{\chi}(t)$ verifies $\mathcal{L}_{c} = \mathcal{L}_{t}$.
\end{theorem}

Another interesting property that this modification permits is that repeating an argument has an even bigger impact: not only it resets its status to acceptable, but it also forces to re-evaluate the whole graph while considering this argument as acceptable first. 

Now if we look at the other formal properties, correctness still holds. Indeed, the characterisation of a $\Sigma_c$-argumentative state with the exogenous events still holds as well. The only difference is that now, the previously stated arguments are dealt with first, and then a new update loop occurs on the whole graph. This is the same for termination. The idea of the proof does not change. The only difference is that instead of breaking down the updating process into two steps, there are now four.

\subsubsection{Preliminary discussion about other semantics.}
Beyond the complete semantics studied in Section~\ref{sec:translation}, we discuss here other semantics and the technical issues they raise.
First, as there is a single grounded extension, constructing a correct transformation for this semantic will result in the first form of completeness, i.e. no matter what the order of enunciation is, the final argumentative state will remain the same: the grounded extension. This is not the behaviour we are aiming for. Now, the preferred semantic can be viewed as global semantic in so far as it is a maximum complete extension. As a consequence, it requires to determine all the complete or admissible extensions in a preliminary step. This is a global property that does not fit our approach. Indeed, as discussed before, our goal is to generate an extension given an order of enunciation. Thus, without a ``local" characterisation, i.e. independently from other semantics, it raises new technical challenges to guarantee that the labelling associated with the final state is a preferred one rather than just a complete one.

Regarding the stable semantics, it is possible to propose a definition of a \emph{$\Sigma_s$-argumentative state} based on the property that a labelling $\mathcal{L}$ is a stable one iff it is complete and $undec(\mathcal{L}) = \emptyset$.

\begin{definition}[$\Sigma_s$-argumentative state]\label{def:stable_state}
    A state~$S(t)$ is a \emph{$\Sigma_s$-argumentative state} if:\\ 
            i) $\forall x, [S(t)\models p_x \wedge i_x \wedge \neg o_x$ \\ \phantom{i) $\forall x$}$\Leftrightarrow S(t)\models p_x \wedge \left(\bigwedge_{y} (cA_{y,x} \wedge \neg i_y \wedge o_y) \vee \neg cA_{y,x} \right)]$;\\
            ii) $\forall x, [S(t)\models p_x \wedge \neg i_x \wedge o_x$ \\ \phantom{ii) $\forall x,$}$ \Leftrightarrow \exists y, S(t)\models  p_x \wedge p_y \wedge cA_{y,x} \wedge i_y \wedge \neg o_y]$; \\
            iii) $\forall x, \left[S(t)\models p_x \wedge ((i_x \wedge \neg o_x) \vee (\neg i_x \wedge o_x)) \right]$.
\end{definition}

The difficulty with the stable semantics and this definition is that there is no guarantee of the existence of at least one extension for each argumentation graph. If we follow the same methodology as the one used in Section~\ref{sec:translation}, then the termination is not guaranteed anymore. Indeed, let us assume that it is possible to prove that a $\Sigma_s$-argumentative state is equivalent to a state where no exogenous event is triggered and that it is correct. Now let us imagine the scenario given a $\Sigma_s$-argumentative state where a stable extension exists, a new argument is enunciated. However, there is no stable extension in the new associated argumentation graph. Then, unless we add a stopping condition, exogenous events will continue to be triggered for each state without stopping as it is impossible to reach this $\Sigma_s$-argumentative state. An in-depth study of these three semantics is left for future work.

\section{Conclusion}
\label{sec:conclusion}

This paper has proposed a formalisation of abstract argumentation systems integrating the order of enunciation of the arguments,
establishing and proving its formal properties for the complete semantics. Its advantage is twofold: it allows us to discriminate between the extensions of a semantics, and it increases the amount of information by saving the history of the whole dialogue in two traces.

This work is a first step in exploring how using ADLs can benefit AAFs for explanations. Indeed, the notion of causality associated with the ADL offers the possibility to give rich information about the argument's acceptability status and justifications about the latter. Therefore, ongoing work aims at formalising Miller's desired explanation properties~\cite{miller_explanation_2018} in the ADL framework, to propose a compliant explanation generation and ordering method. We also aim to study our ``last enunciated last updated" strategy in terms of appropriateness. Lastly, since this work is a theoretical one paving the way for practical applications, experimental evaluations are needed. Therefore, the next step is to perform user studies to assess  the intelligibility and coherence of the generated explanations, in terms of objective understanding and subjective satisfaction.

\bibliographystyle{splncs04}
\bibliography{mybibfile}

\newpage

\section*{Appendix}

\begin{proof}[Proof of Proposition~\ref{lem:impossible_fluents}]
    Let us prove by contradiction that $\forall t \in \mathbb{T}, \forall x$, $S^{\chi}(t) \not\models i_x \wedge o_x$. We define $t_0 = \min\limits_{t \geq 1} \{t \mid \exists x, (S^{\chi}(t-1) \not\models i_x \wedge o_x) \text{ and } (S^{\chi}(t) \models i_x \wedge o_x)\}$. 
    By construction, $\exists x_0, S^{\chi}(t_0-1) \models \neg i_{x_0} \vee \neg o_{x_0}$ and $S^{\chi}(t_0) \models i_{x_0} \wedge o_{x_0}$.

    (i) If $S^{\chi}(t_0-1) \models \neg o_{x_0}$ then to reach $S^{\chi}(t_0)$, $\exists y_0, S^{\chi}(t_0-1) \models tri(\mUnacc^2_{y_0,x_0})$, i.e. $S^{\chi}(t_0-1) \models p_{x_0} \wedge \neg o_{x_0} \wedge cA_{y_0,x_0} \wedge p_{y_0} \wedge i_{y_0}$. Moreover, by definition of $t_0$, $S^{\chi}(t_0-1) \models i_{y_0} \wedge \neg o_{y_0}$. 
    
    \phantom{a} $\bullet$ If $S^{\chi}(t_0-1) \models i_{x_0} \wedge \neg o_{x_0}$, then $S^{\chi}(t_0-1) \models tri(\mUnacc^1_{y_0,x_0})$ which is impossible as $\mUnacc^2_{y_0,x_0} \in E^{\chi}(t_0)$ and $\forall x,y, \mUnacc^1_{y,x} \succ  \mUnacc^2_{y,x}$.

    \phantom{a} $\bullet$ If $S^{\chi}(t_0-1) \models \neg i_{x_0} \wedge \neg o_{x_0}$, then in addition to $\mUnacc^2_{y_0,x_0}$, $\mAcc^2_{x_0}$ also has to trigger so that $x_0$ becomes \emph{in.} Thus, $\mAcc^2_{x_0} \in E^{\chi}(t_0 - 1)$. It means that $S^{\chi}(t_0-1) \models p_{x_0} \wedge \neg i_{x_0} \wedge \left(\bigwedge_{y} (p_y \wedge cA_{y,x_0} \wedge o_y) \vee \neg cA_{y,x_0} \right)$. It especially applies to $y_0$ and we have that $S^{\chi}(t_0-1) \models i_{y_0} \wedge o_{y_0}$. This is a contradiction by construction of $t_0$.

    By contradiction, $S^{\chi}(t_0-1) \not \models \neg o_{x_0}$

    (ii) If $S^{\chi}(t_0-1) \models \neg i_{x_0} \wedge o_{x_0}$ then $\mAcc^2_{x_0} \in E^{\chi}(t_0 - 1)$. Therefore, by definition of $t_0$, $\forall y, S^{\chi}(t_0-1) \models (p_{y} \wedge cA_{y,x_0} \wedge o_{y} \wedge \neg i_{y}) \vee \neg cA_{y,x_0}$. \\ Then, $S^{\chi}(t_0-1) \models p_{x_0} \wedge o_{x_0} \wedge \left(\bigwedge_{y} (p_y \wedge cA_{y,x_0} \wedge \neg i_y) \vee \neg cA_{y,x_0} \right)$, i.e. $S^{\chi}(t_0-1) \models \mAcc^1_{x_0}$. But, $\forall x, \mAcc^1_{x} \succ  \mAcc^2_{x}$. So this is a contradiction.

    We proved by contradiction that such a $t_0$ does not exist. Moreover, for $t =0$ there is no present argument so they are all unacceptable. Therefore, $\forall t \in \mathbb{T}, \forall x$, $S^{\chi}(t) \not\models i_x \wedge o_x$.
\end{proof}

\begin{proof}[Proof of Lemma~\ref{lem:argu_state_tri}]
    Let us start by proving that (i) $\Rightarrow$ (ii). Let $S^{\chi}(t)$ be a $\Sigma_c$-argumentative state and $e \in \mathbb{U}$ be an exogenous event. Let us prove by exhaustion on $e$ that $S^{\chi}(t) \not\models tri(e)$.
    \\ \phantom{a.}$\bullet$ Let $x,y$ such that $e = \mUnacc^1_{y,x}$ and let us suppose that $S^{\chi}(t) \models tri(e)$ i.e. $S^{\chi}(t) \models p_x \wedge i_x \wedge p_y \wedge cA_{y,x} \wedge \neg o_y$. According to Lemma~\ref{lem:impossible_fluents}, $S^{\chi}(t) \models p_x \wedge i_x \wedge \neg o_x \wedge p_y \wedge cA_{y,x} \wedge \neg o_y$. According to (i) of Definition~\ref{def:complete_state}, $S^{\chi}(t) \models p_x \wedge i_x \wedge \neg o_x$ implies that $S^{\chi}(t) \models p_x \wedge cA_{y,x} \wedge \neg i_y \wedge o_y$. So $S^{\chi}(t) \models o_y \wedge \neg o_y$. This is impossible because a state is coherent.
    \\ \phantom{a.}$\bullet$ Let $x,y$ such that $e = \mUnacc^2_{y,x}$ and let us suppose that $S^{\chi}(t) \models tri(e)$ i.e. $S^{\chi}(t) \models p_x \wedge \neg o_x \wedge p_y \wedge cA_{y,x} \wedge i_y$. According to Lemma~\ref{lem:impossible_fluents},  $S^{\chi}(t) \models p_x \wedge \neg o_x \wedge p_y \wedge cA_{y,x} \wedge i_y \wedge \neg o_y$. According to (ii) of Definition~\ref{def:complete_state}, $S^{\chi}(t) \models p_x \wedge p_y \wedge cA_{y,x} \wedge i_y \wedge \neg o_y$ implies that $S^{\chi}(t) \models p_x \wedge \neg i_x \wedge o_x$. So $S^{\chi}(t) \models o_x \wedge \neg o_x$. This is impossible because a state is coherent.
    \\ \phantom{a.}$\bullet$ Let $x$ such that $e = \mAcc^1_{x}$ and let us suppose that $S^{\chi}(t) \models tri(e)$ i.e. $S^{\chi}(t) \models p_x\wedge o_x\wedge \left(\bigwedge_{y} (cA_{y,x} \wedge \neg i_y) \vee \neg cA_{y,x} \right)$. According to Lemma~\ref{lem:impossible_fluents}, $S^{\chi}(t) \models p_x \wedge \neg i_x \wedge o_x$, so according to (ii) of Definition~\ref{def:complete_state}, $\exists y$, $S^{\chi}(t) \models p_x \wedge p_y \wedge cA_{y,x} \wedge i_y \wedge \neg o_y$. Therefore $S^{\chi}(t) \models cA_{y,x} \wedge \neg i_y \wedge i_y$. This is impossible because a state is coherent.
    \\ \phantom{a.}$\bullet$ Let $x$ such that $e = \mAcc^2_{x}$ and let us suppose that $S^{\chi}(t) \models tri(e)$ i.e. $S^{\chi}(t) \models p_x\wedge \neg i_x\wedge \left(\bigwedge_{y} (cA_{y,x} \wedge o_y) \vee \neg cA_{y,x} \right)$. According to Lemma~\ref{lem:impossible_fluents}, $S^{\chi}(t)\models p_x \wedge \left(\bigwedge_{y} (cA_{y,x} \wedge \neg i_y \wedge o_y) \vee \neg cA_{y,x} \right)$. This implies according to (i) of Definition~\ref{def:complete_state}, $S^{\chi}(t)\models p_x \wedge i_x \wedge \neg o_x$. Therefore, $S^{\chi}(t)\models i_x \wedge \neg i_x$. This is impossible because a state is coherent.


    Now let us prove that (ii) $\Rightarrow$ (i). Let $S^{\chi}(t)$ be a state such that $ \forall e \in \mathbb{U}, S^{\chi}(t) \not\models tri(e)$. Let us prove that $S^{\chi}(t)$ is a $\Sigma_c$-argumentative state, i.e. it verifies (i) and (ii) of Definition~\ref{def:complete_state}.
    \\(i) $\Leftarrow:$ Let $x$ such that $S^{\chi}(t)\models p_x \wedge \left(\bigwedge_{y} (cA_{y,x} \wedge \neg i_y \wedge o_y) \vee \neg cA_{y,x} \right)$. It is equivalent to $S^{\chi}(t)\models p_x \wedge \left(\bigwedge_{y} (cA_{y,x} \wedge \neg i_y) \vee \neg cA_{y,x} \right) \wedge \left(\bigwedge_{y} (p_y \wedge cA_{y,x} \wedge \wedge o_y) \vee \neg cA_{y,x} \right)$
    \\ As $\forall e \in \mathbb{U}, S^{\chi}(t) \not\models tri(e)$, then with $e = \mAcc^1_x$ and $e' = \mAcc^2_x $, we have that $S^{\chi}(t) \not\models (p_x \wedge \left(\bigwedge_{y} (cA_{y,x} \wedge \neg i_y) \vee \neg cA_{y,x} \right)) \wedge o_x$ and $S^{\chi}(t) \not\models (p_x \wedge \left(\bigwedge_{y} (cA_{y,x} \wedge  o_y) \vee \neg cA_{y,x} \right)) \wedge \neg i_x$. So $S^{\chi}(t) \not\models o_x$ and $S^{\chi}(t) \not\models \neg i_x$. Therefore, $S^{\chi}(t) \models p_x \wedge i_x \wedge \neg o_x$.
    \\(i) $\Rightarrow:$ Let $x$ such that $S^{\chi}(t) \models p_x \wedge i_x \wedge \neg o_x$. As $\forall e \in \mathbb{U}, S^{\chi}(t) \not\models tri(e)$, then $\forall y, S^{\chi}(t) \not\models tri(\mUnacc^1_{y,x})$ and $S^{\chi}(t) \not\models tri(\mUnacc^2_{y,x})$. Therefore, $\forall y, S^{\chi}(t) \not\models cA_{y,x} \wedge p_y \wedge \neg o_y$ and $S^{\chi}(t) \not\models cA_{y,x} \wedge p_y \wedge i_y$. If $S^{\chi}(t) \models p_y$ then $S^{\chi}(t) \models o_y \vee \neg cA_{y,x}$ and $S^{\chi}(t) \models \neg i_y \vee \neg cA_{y,x}$. If $S^{\chi}(t) \models \neg p_y$ then $y$ has never been enunciated and so is unacceptable i.e$S^{\chi}(t) \models \neg i_y \wedge o_y$. Therefore, $\forall y, S^{\chi}(t) \models (cA_{y,x} \wedge o_y \wedge \neg i_y) \vee \neg cA_{y,x}$. So $S^{\chi}(t)\models p_x \wedge \left(\bigwedge_{y} (cA_{y,x} \wedge \neg i_y \wedge o_y) \vee \neg cA_{y,x} \right)$
    \\(ii) $\Rightarrow:$ Let $x$ such that $S^{\chi}(t)\models p_x \wedge \neg i_x \wedge o_x$. As $S^{\chi}(t)\not\models tri(\mAcc^1_x)$ then $\exists y$ such that $S^{\chi}(t)\models \neg ((cA_{y,x} \wedge \neg i_y) \vee \neg cA_{y,x})$ i.e. $S^{\chi}(t)\models cA_{y,x} \wedge i_y$. According to Lemma~\ref{lem:impossible_fluents}, $S^{\chi}(t)\models cA_{y,x} \wedge i_y \wedge \neg o_y$. Moreover, as $S^{\chi}(t)\models i_y \wedge \neg o_y$, $S^{\chi}(t)\models p_y$. Therefore, $\exists y, S^{\chi}(t)\models p_x \wedge p_y \wedge cA_{y,x} \wedge i_y \wedge \neg o_y$.
    \\(ii) $\Leftarrow:$ Let $x$ such that $\exists y, S^{\chi}(t)\models p_x \wedge p_y \wedge cA_{y,x} \wedge i_y \wedge \neg o_x$. As $S^{\chi}(t)\not\models tri(\mUnacc^2_{y,x})$ then $S^{\chi}(t)\models o_x$. According to Lemma~\ref{lem:impossible_fluents}, $S^{\chi}(t)\models \wedge \neg i_x \wedge \neg o_x$. Therefore, $S^{\chi}(t)\models p_x \wedge \neg i_x \wedge o_x$.

    
\end{proof}

\begin{proof}[Termination]
    Let $S^{\chi}(t)$ be a $\Sigma_c$-argumentative state such that $E(t) = \{enunciate_x\}$. Let us prove that $\exists t' \geq t$, $S^{\chi}(t')$ is a $\Sigma_c$-argumentative state.

    As $E(t) = \{ enunciate_x \}$ then $\forall f \in \mathbb{F} \setminus \{p_x, i_x, o_x\}$ if $S^{\chi}(t) \models f$, $S^{\chi}(t+1) \models f$. Moreover, $S^{\chi}(t+1) \models e\ff(enunciate_x) = p_x \wedge i_x \wedge \neg o_x$. According to Lemma~\ref{lem:argu_state_tri}, as $S^{\chi}(t)$ is a $\Sigma_c$-argumentative state then $\forall e \in \mathbb{U}$, $S^{\chi}(t) \not\models tri(e)$. So, $\forall e \in \mathbb{U}\setminus \{\mUnacc^1_{y,x}, \mUnacc^2_{y,x}, \mUnacc^1_{x,y}, \mUnacc^2_{x,y} \mid \forall y \}$, $S^{\chi}(t) \not\models tri(e)$.
    \\ As $\forall x_1,y_1,x_2,y_2, \mUnacc^1_{y_1,x_1} \succ_\mathbb{E} \mUnacc^2_{y_2,x_2}$, then if $\exists y$, $S^{\chi}(t+1) \models tri(\mUnacc^1_{y,x}) \vee tri(\mUnacc^1_{x,y})$ then $E(t+1) = \{\mUnacc^1_{y,x} \mid S^{\chi}(t+1) \models tri(\mUnacc^1_{y,x}) \} \cup \{\mUnacc^1_{x,y} \mid S^{\chi}(t+1) \models tri(\mUnacc^1_{x,y}) \}$ i.e. all the exogenous events of the form $\mUnacc^1_{b,a}$ that can be triggered are triggering. 
    \\ Then, as $e\ff(\mUnacc^1_{b,a} = \neg i_a$,  $\forall e \in E(t+2)$, $e$ is of the form $\mUnacc^2$ or $\mAcc^1$. As $\forall a,b,c, \mAcc^1_c \succ_\mathbb{E} \mUnacc^2_{b,a}$, if $\exists z, S^{\chi}(t+2) \models \mAcc^1_z$ then $\forall e \in E(t+2)$, $e$ is of the form $\mAcc^1$. Therefore, as $e\ff(I^1_z) = \neg o_x$, again the only events that can be triggered are of the form $\mUnacc^1$ or $\mUnacc^2$. Then because of the priority rule $R_3$, we are in the same scenario as before. Therefore, the only event that are triggering at each state are $\mUnacc^1$ followed by $\mAcc^1$. As long as this chain is not broken, no other events can be triggered. The effects of this two class of exogenous events being not contradictory, it can be triggered at most as much as the number of arguments. The number of arguments being finite, there exists $t_1 \geq t$ such that no events of this type are triggered.

    Now, as the effects of those events cannot trigger $\mAcc^2$, then the only exogenous events that can be triggered are $\mUnacc^2_{y,x}$ with $x$ the argument that has just been enunciated. If it is not triggering, then we are in a state where no more exogenous events can trigger so according to Lemma~\ref{lem:argu_state_tri}, $S^{\chi}(t_1)$ is a $\Sigma_c$-argumentative state.

    Otherwise, $S^{\chi}(t_1 + 1) \models o_x$. Along the same line, the only events that can be triggered in this state are of the form $\mAcc^2$. If nothing triggers then we are in a $\Sigma_c$-argumentative state, else $\exists y, S^{\chi}(t_1 + 2) \models i_y$. \\ Again, because some arguments~$(y_i)$ became \emph{in}, the only events that can trigger are $\mUnacc^2_{y_i,z_i}$. Finally, there is a chain of triggering with $\mUnacc^2$ being followed by $\mAcc^2$. As their effect are not cancelling each other, this can happened at most as much as the number of arguments. This number being finite, $\exists t_2 \geq t_1$ such that $\forall e \in \mathbb{U}, S^{\chi}(t_2) \not\models t_2$. According to Lemma~\ref{lem:argu_state_tri}, $S^{\chi}(t_2)$ is a $\Sigma_c$-argumentative state.
        
\end{proof}

\begin{proof}[Correctness]
    Let $S{\chi}(t)$ be a $\Sigma_c$-argumentative state and $\mathcal{L}_t$ its associated labelling. Let us prove that $\mathcal{L}_t$ is a complete labelling of $AF = (Ar,\mathcal{R})$ : \\ (i) $\forall x \in Ar$, $\mathcal{L}_t(x) = IN \Leftrightarrow \forall y \in Att_x, \mathcal{L}_t(y) = OUT$ \\ (ii) $\forall x \in Ar$, $\mathcal{L}_t(x) = OUT \Leftrightarrow \exists y \in Att_x, \mathcal{L}_t(y) = IN$

    (i): Let $x \in Ar$ be an argument such that $\mathcal{L}_t(x) = IN$. By definition of $\mathcal{L}_t$, $S^{\chi}(t) \models p_x \wedge i_x \wedge \neg o_x$. As $S^{\chi}(t)$ is a $\Sigma_c$-argumentative state, it is equivalent to $S^{\chi}(t)\models p_x \wedge \left(\bigwedge_{y} (cA_{y,x} \wedge \neg i_y \wedge o_y) \vee \neg cA_{y,x} \right)$. Therefore, $\forall y \in Att_x, S^{\chi}(t)\models \neg i_y \wedge o_y$ i.e. by definition of $\mathcal{L}_t$, $\mathcal{L}_t(y) = OUT$. 

    (ii): Let $x \in Ar$ be an argument such that $\mathcal{L}_t(x) = OUT$. By definition of $\mathcal{L}_t$, $S^{\chi}(t) \models p_x \wedge \neg i_x \wedge o_x$. As $S^{\chi}(t)$ is a $\Sigma_c$-argumentative state, it is equivalent to $\exists y, S^{\chi}(t) \models p_x \wedge p_y \wedge cA_{y,x} \wedge i_y \wedge \neg o_y$. Therefore, $\exists y \in Att_x, S^{\chi}(t) \models i_y \wedge \neg o_y$ i.e. by definition of $\mathcal{L}_t$, $\mathcal{L}_t(y) = IN$.
    \end{proof}

\begin{proof}[Completeness of the modified transformation]
    Let $\mathcal{L}_c$ be a complete labelling for $AF = (\mathcal{A},\mathcal{R})$. In this proof, we will build a sequence $\varsigma_c$ that leads to this complete labelling using the modified transformation.

    At time $t = 0$, all the arguments labelled $OUT$ by $\mathcal{L}_c$ are enunciated. Then when an argumentative state is reached at $t_{IN}$, we enunciate all the arguments labelled $IN$ by $\mathcal{L}_c$. Based on the characterisation of a complete labelling provided in~\cite{caminada2017argumentation}, $b$ is labelled $OUT$ iff $\exists a \in \mathcal{A}$ that is labelled $IN$. Moreover, $a$ is labelled $IN$ iff $\forall b \in Att_a$, $b$ is labelled $OUT$. Therefore, $\forall b$ labelled $OUT, \exists a$ such that $S^\chi(t_{IN}+1) \models tri(\mUnacc^1_{a,b})$. As arguments labelled $IN$ have just been stated, their acceptability cannot be changed because there are other updates that can be triggered. After that, for the same reason, we only have $S^\chi(t_{IN} +2) \models tri(\mUnacc^1_{a,b})$. We reach an argumentative state $S^\chi(t_{U})$ in which arguments labelled $IN$ are acceptable and the ones labelled $OUT$ are unacceptable.

    Finally, we enunciate all the arguments labelled $UNDEC$. As $\forall b \in Att_a$, $b$ is labelled $OUT$ and all the other arguments are unacceptable, the only events that can be triggered are concerning the new arguments. According to~\cite{caminada2017argumentation}, an argument $c$ is labelled $UNDEC$ iff $\forall x \in Att_c$, $x$ is not labelled $IN$ and $\exists y \in Att_c$ such that $y$ is not labelled $OUT$. Therefore, $\forall c$ labelled $UNDEC, \exists y$ that is labelled $UNDEC$ such that $S^\chi(t_{U}) \models tri(\mUnacc^1_{y,c})$. Applying this for all $UNDEC$ arguments, we reach a new state $S^\chi(t_U+1)$ in which all these arguments are $\neg i \wedge \neg o$, i.e. $UNDEC$ for the associated labelling.
\end{proof}

\end{document}